\documentclass[conference]{IEEEtran}

\IEEEoverridecommandlockouts
\usepackage{mathtools, amsthm, amsmath,amssymb,amsfonts, dsfont}
\usepackage{multirow, multicol, comment}
\usepackage{stfloats, booktabs,threeparttable, tabularx}
\usepackage{mathrsfs,bookmark}
\usepackage{makecell}
\usepackage{bm}
\usepackage{algorithm}
\usepackage{algpseudocode}
\usepackage{graphicx}
\usepackage{textcomp}
\usepackage{cite}
\usepackage{xspace}
\usepackage{cuted} 
\usepackage{xcolor}
\usepackage{lipsum}
\usepackage{svg}

\usepackage{subfigure}

\newtheorem{theorem}{Theorem}

\newtheorem{definition}{Definition}

\newtheorem{remark}{Remark}

\DeclareMathOperator{\tr}{tr}

\newcommand{\st}[1]{\text{s.t.}}

\newcommand{\vw}{\boldsymbol{w}}

\newcommand{\vx}{\boldsymbol{x}}

\newcommand{\vy}{\boldsymbol{y}}
\newcommand{\vz}{\boldsymbol{z}}

\newcommand{\vn}{\boldsymbol{n}}

\newcommand{\vv}{\boldsymbol{v}}

\newcommand{\vg}{\boldsymbol{g}}
\newcommand{\vu}{\boldsymbol{u}}

\newcommand{\bA}{\mathbf{A}}
\newcommand{\bB}{\mathbf{B}}

\begin{document}

\title{Federated Low-Rank Adaptation with Differential Privacy over Wireless Networks
\\
}

\author{\IEEEauthorblockN{Tianqu Kang, Zixin Wang, Hengtao He, Jun Zhang, Shenghui Song, and Khaled B. Letaief}\\
\IEEEauthorblockA{
Dept. of ECE, The Hong Kong University of Science and Technology, Hong Kong \\
Email: tkang@connect.ust.hk, \{eewangzx, eehthe, eejzhang, eeshsong, eekhaled\}@ust.hk
}}

\maketitle

\begin{abstract}
Fine-tuning large pre-trained foundation models (FMs) on distributed edge devices presents considerable computational and privacy challenges. Federated fine-tuning (FedFT) mitigates some privacy issues by facilitating collaborative model training without the need to share raw data. To lessen the computational burden on resource-limited devices, combining low-rank adaptation (LoRA) with federated learning enables parameter-efficient fine-tuning. Additionally, the split FedFT architecture partitions an FM between edge devices and a central server, reducing the necessity for complete model deployment on individual devices. However, the risk of privacy eavesdropping attacks in FedFT remains a concern, particularly in sensitive areas such as healthcare and finance. In this paper, we propose a split FedFT framework with differential privacy (DP) over wireless networks, where the inherent wireless channel noise in the uplink transmission is utilized to achieve DP guarantees without adding an extra artificial noise. We shall investigate the impact of the wireless noise on convergence performance of the proposed framework. We will also show that by updating only one of the low-rank matrices in the split FedFT with DP, the proposed method can mitigate the noise amplification effect. Simulation results will demonstrate that the proposed framework achieves higher accuracy under strict privacy budgets compared to baseline methods.

\end{abstract}


\section{{\color{black}Introduction}} \label{intro}

The rapid advancement of artificial intelligence (AI) has enabled the development of powerful pre-trained foundation models (FMs), such as large language models (LLMs) and large vision models (LVMs), which demonstrate remarkable capabilities across diverse domains \cite{10384606, 10183789, 9606720}. Fine-tuning these models for specific tasks often requires access to domain-specific data distributed across numerous edge devices in real-world applications.

Federated Learning (FL) has emerged as a promising paradigm for decentralized model training, enabling multiple edge devices to collaboratively learn a shared model without exchanging raw data \cite{LiuTWC2022, 10558823, bnn}. In the context of fine-tuning FMs, federated fine-tuning (FedFT) over wireless networks allows devices to adapt pre-trained models to their local data, leveraging collective knowledge \cite{10558823}. However, fine-tuning full FMs on resource-constrained edge devices is often impractical due to the substantial computational and memory demands. 
To reduce these demands, parameter-efficient fine-tuning (PEFT) methods, such as Low-Rank Adaptation (LoRA), reparameterize weight updates using low-rank matrices to reduce the number of trainable parameters \cite{hu2022lora}. However, even with LoRA, deploying full models on edge devices may still exceed the capacity of edge devices. A split FedFT architecture, proposed in \cite{wang2024federated}, addresses this issue by distributing components of the FM across edge devices and a central server. In this architecture, embedding and task-specific modules are placed on devices, while the computationally intensive encoder is deployed on the server. 
However, privacy concerns due to the potential risks associated with an untrusted server have not been considered. Specifically, gradient inversion attacks can exploit shared gradients transmitted during training to reconstruct sensitive data \cite{Hatamizadeh_2022_CVPR}, posing significant privacy threats to user information.

Differential Privacy (DP) has been proposed as a solution to protect against such adversarial attacks by adding artificial noise to shared gradients \cite{dwork2014algorithmic, Tianqu}. In this paper, we incorporate DP into the split FedFT framework to enhance data protection. 
We leverage inherent channel noise in wireless networks as a natural DP mechanism, reducing the need for artificial noise at the client side \cite{FreePrivacy}. We futher establish the relationship between privacy loss and wireless fading channels and introduce a privacy-aware power control policy to ensure robust privacy protection.

Integrating DP into the split FedFT, however, presents additional challenges. 
The cascaded architecture of the FM can amplify noise as gradients propagate through multiple layers, especially within low-rank matrix updates, destabilizing model training and degrading performance \cite{sun2024improving}. To address this issue, we propose a modified LoRA-based FedFT architecture that reduces noise amplification without compromising privacy protection. By updating only one low-rank matrix and fixing another matrix as a scaled orthonormal matrix, we reduce noise amplification by eliminating higher-order noise terms during backpropagation and stabilize the energy of noise in weight updates. Experimental results will demonstrate the superior performance in model accuracy under strict privacy budgets compared to baseline methods. The proposed approach is promising for the practical deployment of large-scale AI models on resource-constrained edge devices with robust privacy guarantees.

\section{System Model and Preliminary} \label{preliminary}


{\color{black}In this section, we present the split FedFT framework
in Section \ref{subsec:system}, followed by the integration of DP in Section \ref{subsec:fft}. We then describe the communication model in Section \ref{subsec: comm}, focusing on how edge devices transmit local gradient deviations via uncoded transmission in a Time Division Multiple Access (TDMA) system.}

\begin{figure}[t]
    \centering
    \includegraphics[width=0.48\textwidth]{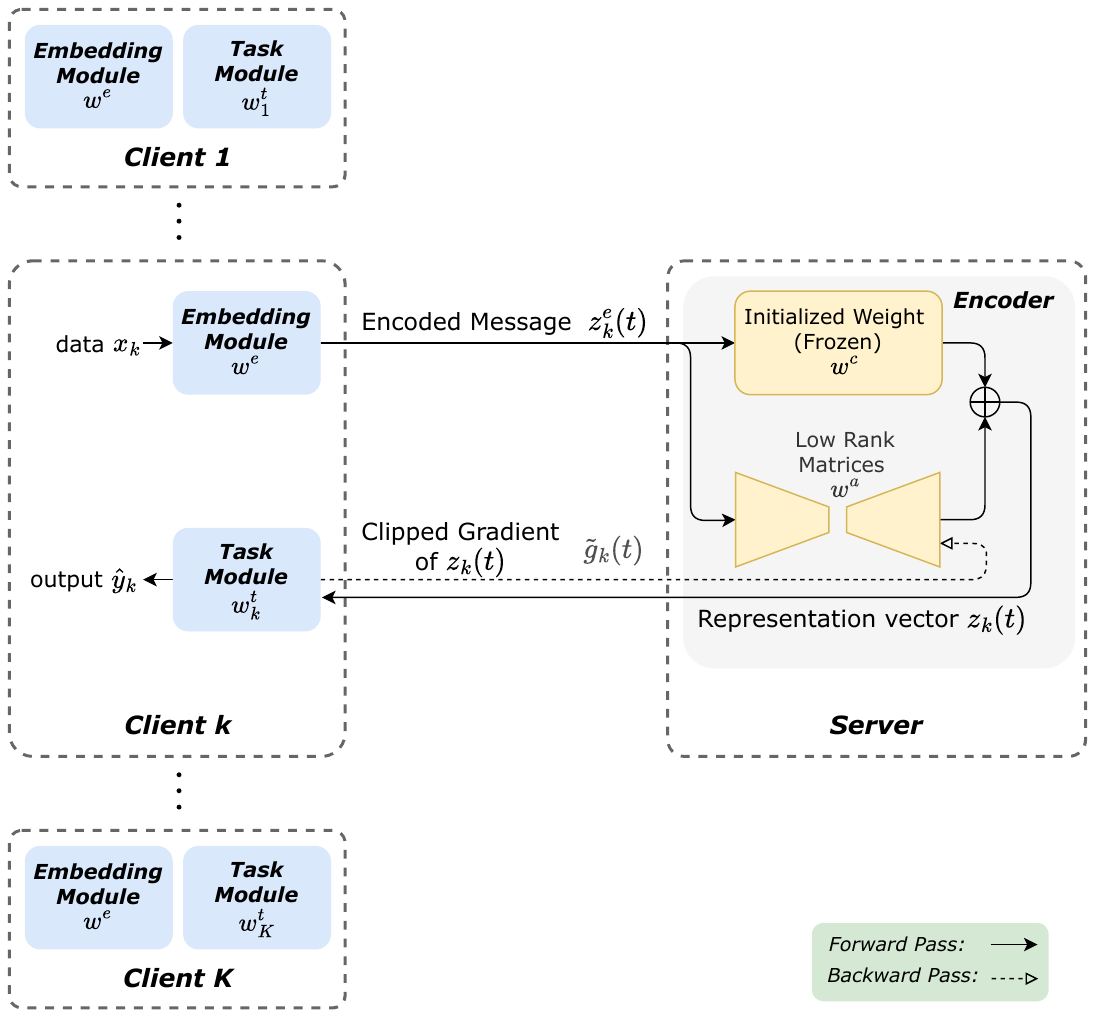}
    \caption{{\color{black}System model of the proposed LoRA-based FedFT, focusing on the communication between the $k$-th edge device and the edge server. The computation-intensive encoder resides at the edge server, while the embedding and task modules are on the edge devices. The forward pass is represented by solid black arrows, while the backward pass is shown with dashed arrows.}}
    \label{systemfig}
\end{figure}
{\color{black}

\subsection{{\color{black}Split LoRA-based FedFT Framework}} \label{subsec:system}
As illustrated in Fig.~\ref{systemfig}, we consider a split LoRA-based FedFT framework over wireless networks. 
Specifically, a single-antenna edge server coordinates $K$ single-antenna edge devices, denoted by $\mathcal{K} = \{1, 2, \ldots, K\}$, to collaboratively fine-tune a global model using their local datasets, which is denoted by $\mathcal{D}_k = \{(\vx_m, \vy_m)\}_{m=1}^{M}$. 
Following the architecture in \cite{wang2024federated}, the pre-trained FM $\vw^{\rm f}$ is divided into three components: the embedding module $\vw^{\rm e}$, the encoder module $\vw^{\rm c}$, and the task module $\vw^{\rm t}$.
 The embedding and task modules are deployed on the edge devices while the computation intensive encoder resides at the edge server. In particular, by applying LoRA to the encoder module at the server, a set of trainable low-rank matrices $\vw^{\rm a}$ is parallelly added to the encoder module while keeping the original encoder parameters unchanged.
{\color{black}The goal of the LoRA-based FedFT framework is to refine a set of the task-specific parameters $\{\vw_{k}^{\rm t}\}_{k=1}^{K}$ and the shared low-rank matrices $\vw^{\rm a}$, collectively denoted as $\vw = \{\vw^{\rm a}, \{\vw_{k}^{\rm t}\}_{k=1}^{K}\}$, to minimize the global loss function }$F(\cdot)$. The optimization problem is formulated as }
\begin{equation*}
    \begin{aligned}\label{EFT_obj} 
        \vw^{\star} & = \underset{\{\vw_k^{\rm t}\}_{k=1}^K, \vw^{\rm a}}{\arg\min} \; F\left(\vw; \vw^{\rm f}, \{{\mathcal{D}_k}\}_{k=1}^K\right) \\
        & =  \underset{\{\vw_k^{\rm t}\}_{k=1}^K, \vw^{\rm a}}{\arg\min}\,
        \sum_{k=1}^{K} \frac{|{\mathcal{D}_k}|}{{\sum_{k=1}^K{|\mathcal{D}_k}|}}
        \mathcal{L}_k\left(\vw_k^{\rm t}, \vw^{\rm a}; \vw^{\rm f}, \mathcal{D}_k\right),
    \end{aligned}
\end{equation*}
where {\color{black}$\mathcal{L}_k(\cdot)$} and $|{\mathcal{D}_k}|$ denote the local loss function and the size of the local dataset at device $k$, respectively.

\subsection{Split LoRA-based FedFT with DP}\label{subsec:fft}
While raw data privacy is preserved on edge devices in the LoRA-based FedFT framework, there remains a risk of privacy leakage through gradient information during aggregation. To address this issue, we integrate DP into the split FedFT framework. We first introduce the definitions of $(\varepsilon, \delta)$-DP and $\ell_2$ sensitivity, followed by a description of the complete training process incorporating DP.

\begin{definition} [$(\varepsilon, \delta)$-DP and  $\ell_2$ sensitivity]
An algorithm $\mathcal{M}: \mathcal{D}\rightarrow \mathcal{S}$ is $(\epsilon, \delta)$-DP if, for all neighboring databases $D, D' \in \mathcal{D}$  differing by a single record and all $S \subseteq \mathcal{S}$,
\begin{equation}
    \Pr[\mathcal{M}(D)\in S] \leq e^\varepsilon \Pr[\mathcal{M}(D')\in S] +\delta,
\end{equation}
where $\varepsilon \geq 0$ and $0<\delta<1$.

The $\ell_2$ sensitivity of a function $f$, denoted by $\Delta_2$ can be defined as 
\begin{equation}
    \Delta_2 = \max_{D,D'} \|f(D)-f(D')\|_2.
\end{equation}
\end{definition}


The training process of the split FedFT with a DP framework can be performed by the following steps in each communication round.
{\color{black}
\begin{enumerate}
    \item [$\bullet$]\textbf{Feature Extraction and Transmission}: Each edge device $k$ encodes its local data $\vx_k$ using the embedding module $\vw^{\rm e}$ {\color{black}to obtain the encoded message $\vz^{\rm e}_k(t)$} and transmits it to the edge server.

    \item [$\bullet$] \textbf{Server-side Processing and Feedback}: The edge server receives $\vz^{\rm e}_k(t)$ and processes it using the encoder $\vw^{\rm c}$ and the low-rank matrices $\vw^{\rm a}(t)$ to compute representations $\vz_k(t)$, which are then sent back to the devices.

    \item [$\bullet$] \textbf{Local Task Module Update}: Each device updates its task module $\vw_k^{\rm t}$ using $\vz_k(t)$, and computes the gradient deviation of the local loss with respect to $\vz_k(t)$
        \begin{equation}
           {\color{black} \vg_k(t)} = \frac{\partial \mathcal{L}_k}{\partial \vz_k(t)}.
        \end{equation}

    \item [$\bullet$] \textbf{Differential Privacy Processing \footnote{Both $\vz^{\rm e}_k(t)$ and $\color{black} \vg_k(t)$ can potentially reveal information about the local data. This work focuses on protecting $\color{black} \vg_k(t)$, specifically to protect label information, as labels often contain highly sensitive details (e.g., a diagnosis in medical settings or a legal decision) that may be more privacy-critical than feature data in certain applications.}}: Each device clips its gradient deviation $\vg_k(t)$ so that its $\ell_2$ norm does not exceed $C$. That is,
    \begin{equation}
        \tilde{\vg}_k(t) = \frac{C}{\max(C, \|\vg_k(t)\|_2)} \vg_k(t),
    \end{equation}
    which is then transmitted to the edge server in an uncoded manner.

    \item [$\bullet$] \textbf{Global Model Update}: 
    According to the received noisy gradient deviations $\{\tilde{\vg}_k(t)\}$,  the server computes the gradient with respect to the low-rank matrices using the chain rule,
        \begin{equation}\label{eq:gradient}
            \nabla_{\vw^{\rm a}} \mathcal{L}_k(t) = \tilde{\vg}_k(t) \frac{\partial \vz_k(t)}{\partial \vw^{\rm a}}.
        \end{equation} The server aggregates the gradients from all devices as follows,
        \begin{equation}
            \vg^{\rm a}(t) = \sum_{k=1}^{K} \frac{|{\mathcal{D}_k}|}{{\sum_{k=1}^K{|\mathcal{D}_k}|}} \nabla_{\vw^{\rm a}} \mathcal{L}_k(t).
        \end{equation} 
        Finally, the global low-rank matrices are updated as
        \begin{equation}
            \vw^{\rm a}(t+1) = \vw^{\rm a}(t) - \eta \vg^{\rm a}(t),
        \end{equation}
        where $\eta$ is the learning rate.
\end{enumerate}
}

\subsection{Communication Model} \label{subsec: comm}
With the split LoRA-based FedFT framework and DP integration established, we now consider the communication model of the wireless network. 
We assume a block flat-fading channel model, where the channel coefficients remain constant within each block but vary between blocks. Let \( h_k(t) \geq 0 \) denotes the channel coefficient between edge device $k$ and the edge server in the uplink, and \( \boldsymbol{n}_k(t) \sim \mathcal{N}(0, N_0 \mathbf{I}) \) denote the additive white Gaussian noise with zero mean and single-sided noise power spectral density $N_0$. We assume perfect channel state information is available at the edge devices, and the effective channel coefficients are considered \textit{real and non-negative} to simplify the privacy analysis \cite{FreePrivacy}.
All edge devices time-share the channel via TDMA, where each device is scheduled sequentially to transmit its local gradient deviations $\tilde{\vg}_k(t)$ to the edge server via uncoded transmission.

The received gradient deviations at the edge server can be represented by 
\begin{equation} \label{eq:bpuplinkTransmission}
    \boldsymbol{y}_k(t) = h_k(t) \alpha_k^{(t)} \tilde{\boldsymbol{g}}_k(t) + \boldsymbol{n}_k(t),
\end{equation}
where \( \alpha_k^{(t)} \) denotes the {\color{black}scaling factor} of edge device $k$.
Besides, the transmit power of each device \(k\) are constrained by
\begin{equation}\label{eq:power_constraint}
    \mathbb{E}\left[ \left\| \alpha_k^{(t)}\tilde{\boldsymbol{g}}_k(t) \right\|^2 \right] 
    = \left(\alpha_k^{(t)}C\right)^2
    \leq P_k,
\end{equation}
where $C$ is the clipping threshold defined in Seciton \ref{subsec:fft}. The associated signal to noise ratio (SNR) of edge device $k$ is given by
\begin{equation}\label{eq:snr}
    {\color{black}\text{SNR}_k = \frac{\left(\alpha_k^{(t)}C\right)^2}{d N_0}.}
\end{equation}
Then, the server scales the received signal by \( h_k(t) \) to estimate the gradient as,
\begin{equation}\label{eq:gradient_estimation}
    \hat{\vg}_k(t) = \alpha_k^{(t)} \tilde{\vg}_k(t) + \frac{\vn_k(t)}{h_k(t)},
\end{equation}
where the effective noise 
$\frac{\vn_k(t)}{h_k(t)}$
also follows Gaussian distribution with variance 
\( \sigma_{\text{eff}}^2 = {N_0}/{h_k^2(t)} \).

\section{Proposed Method and Analysis} \label{method}
{\color{black}In this section, we first analyze the performance of DP in the split FedFT with LoRA over wireless networks, followed by developping a privacy-aware power control strategy under the stringent privacy constraint. Additionally, we investigate the impact of channel noise amplification, and propose a novel initialization strategy to enhance stability of the split FedFT syste with DP.}

\subsection{Privacy Analysis}
\label{subsec:dp}
{\color{black}
To achieve \emph{local differential privacy (LDP)}, i.e., ensuring the data privacy of each edge device under an untrusted edge server, the channel noise is leveraged for privacy preservation at the edge devices. However, due to the cascaded nature of the split FedFT system, the added noise of the gradient deviations can be amplified in \eqref{eq:gradient}. 
Specifically, the noise component \( \frac{\vn_k(t)}{h_k(t)} \) in \( \hat{\vg}_k(t) \) is propagated through the Jacobian matrix \( \frac{\partial \vz_k(t)}{\partial \vw^{\rm a}} \). Let \( \sigma_{\text{min}} \) and \( \sigma_{\text{max}} \) denote the smallest and largest singular values of \( \frac{\partial \vz_k(t)}{\partial \vw^{\rm a}} \), respectively, where the condition number \( \kappa \) is defined as,
\begin{equation}
    \kappa = \frac{\sigma_{\text{max}}}{\sigma_{\text{min}}}.
\end{equation}
Therefore, the SNR after propagation is bounded by,
\begin{equation}
    \frac{1}{\kappa^2} \cdot \text{SNR}_k \leq \text{SNR}_{\text{after}} \leq \kappa^2 \cdot \text{SNR}_k,
\end{equation}
where \( \text{SNR}_k \) is the SNR of the received gradient \( \hat{\vg}_k(t) \), and \( \text{SNR}_{\text{after}} \) is the SNR of the gradient after propagation through the Jacobian.
A large condition number \( \kappa \) implies that the Jacobian \( \frac{\partial \vz_k(t)}{\partial \vw^{\rm a}} \) can significantly amplify the noise, leading to a degradation of the SNR and performance loss. 

To characterize the impact of channel noise on the learning performance with respect to the SNR of the received gradient \( \hat{\vg}_k(t) \), we first derive the following theorem with respect to the effective noise scale \( \sigma_{\text{eff}} \) based on the \emph{Gaussian mechanism}.

\begin{theorem} \label{thm:privacy}
   Upon the fading channel with channel noise \( \vn_k(t) \) and training epoch $T$, there exist constants $c_1$ that the gradient transmission of edge device $k$ satisfies \((\varepsilon, \delta)\)-local differential privacy, such that
\begin{equation}
    \varepsilon = c_1 h_k(t) \sqrt{T \ln(1/\delta)} \cdot \sqrt{ \text{SNR}_k \cdot d }, \quad \delta > 0.
\end{equation}
\end{theorem}

\begin{proof}
Applying the Moments Accountant theorem over \( T \) iterations with sensitivity as $\alpha_k^{(t)} C$ \cite{dwork2014algorithmic, RenyiDP}, $\exists \; c_1$ such that
\[
    \varepsilon = \frac{c_1\alpha_k^{(t)} C}{\sigma_{\text{eff}}}\sqrt{T \ln(1/\delta)}  .
\]
According to \eqref{eq:snr}, \( \alpha_k^{(t)} \) can be represented by
\begin{equation}
    \alpha_k^{(t)} = \sqrt{ \frac{ \text{SNR}_k \cdot d N_0 }{ C^2 } }.
\end{equation}
Substituting \( \alpha_k \) and \( \sigma_{\text{eff}} \), we get:
\begin{equation}
    \epsilon = c_1 h_k(t) \sqrt{T \ln(1/\delta)} \cdot \sqrt{ \text{SNR}_k \cdot d }.
\end{equation}
\end{proof}

}
\subsection{Privacy-Aware Power Control}
\label{subsec:power_control}

{\color{black}
Recall that Theorem \ref{thm:privacy} demonstrates the relationship between the desired privacy budget $\varepsilon$ and the SNR of the received gradient \( \hat{\vg}_k(t) \) with respect to the channel coefficients $\{h_k(t)\}$ and the scaling factor $\{\alpha_k\}$. In particular, given a privacy requirement $\varepsilon_0$, the standard deviation of the effective noise \( \sigma_{\text{eff}} \) shall satisfy
\begin{equation}
    \sigma_{\text{eff}} = \frac{\sqrt{N_0}}{h_k(t)} \geq \frac{c_1\alpha_k^{(t)} C}{\varepsilon_0}\sqrt{T \ln(1/\delta)} .
\end{equation} 
Recall the transmit power constraint in \eqref{eq:power_constraint}, the scaling factor \(\alpha_k^{(t)}\) shall be upper bounded by
\begin{equation}
    0\leq\alpha_k^{(t)} \leq \min\left\{ 
        \frac{\varepsilon_0\sqrt{N_0}}{c_1Ch_k(t)\sqrt{T \ln(1/\delta)}},
    \frac{\sqrt{P}}{C} \right\}.
\end{equation}

Thus, a privacy-aware power control strategy can be obtained by maximizing the SNR of the received gradient \( \hat{\vg}_k(t) \), where the transmission scaling factor \(\alpha_k^{(t)}\) should be chosen as,
\begin{equation}
    \alpha_k^{(t)} = \min\left\{ \frac{\varepsilon_0\sqrt{N_0}}{c_1Ch_k(t)\sqrt{T \ln(1/\delta)}}, \frac{\sqrt{P}}{C} \right\}.
\end{equation}

By appropriately adjusting the scaling factor \(\alpha_k^{(t)}\), each edge device can obtain the necessary LDP guarantees. This approach eliminates the need to add artificial noise, improving the convergence performance of the split FedFT system while still ensuring privacy.
}
\begin{remark}
    If $\alpha_k^{(t)} = \frac{\sqrt{P}}{C}$, the transmit power is fully utilized, and the devices can rely solely on the channel noise to provide the necessary privacy guarantees. Otherwise, the transmit power need to be adjusted to meet the privacy requirement.
\end{remark}

\subsection{Challenges in Achieving DP with LoRA in FL}
\label{subsec:challenge}
{\color{black}In this subsection, we further investigate the noise amplification effect due to the cascaded nature of LoRA over the network. 
First, we define the added low-rank matrices as {\color{black}\( \mathbf{A}^{(i)} \in \mathbb{R}^{r \times d} \), \( \mathbf{B}^{(i)} \in \mathbb{R}^{d \times r} \)}, with \( r \ll d \) and the original weight of FMs is \( \vw^{\rm c}_i \in \mathbb{R}^{d \times d} \).

In each training round, the adjusted weights for layer \( i \) are}
\begin{equation}
    \vw^{(i)} = \vw^{\rm c}_i + \Delta \vw^{(i)}, \text{where } \Delta \vw^{(i)} = \mathbf{B}^{(i)} \mathbf{A}^{(i)}.
\end{equation}
For edge device \( k \), the output of layer \( i \) is
\begin{equation}
    \vv_k^{(i)}(t) = \vw^{(i)} \vu_k^{(i)}(t) = \left( \vw^{\rm c}_i + \mathbf{B}^{(i)} \mathbf{A}^{(i)}\right) \vu_k^{(i)}(t),
\end{equation}
where \( \vu_k^{(i)}(t) \) is the input to layer \( i \).

To better illustrate this, we plot part of the computation graph of layer \( i \) in Fig.~\ref{fig:computation_graph}. The gradient of the loss function \( \mathcal{L}_k(t) \) with respect to \( \vv_k^{(i)}(t) \) is given by:

\begin{equation}
    \nabla_{\vv_k^{(i)}(t)} \mathcal{L}_k(t) = \left(\alpha_k \tilde{\vg}_k(t) + \frac{\vn_k(t)}{h_k(t)}\right)\frac{\partial\vz_k(t)}{\partial \vv_k^{(i)}(t)}.
\end{equation}
To analyze the impact of noise, we decompose the gradient into a deterministic component and a noise component. Specifically, we define:
\begin{align}
    \vg_{\vv}^{(i)} &= \alpha_k \tilde{\vg}_k(t) \frac{\partial \vz_k(t)}{\partial \vv_k^{(i)}(t)}, \\
    \vn_{\vv}^{(i)} &= \frac{\vn_k(t)}{h_k(t)}\frac{\partial \vz_k(t)}{\partial \vv_k^{(i)}(t)}.
\end{align}
where \( \vg_{\vv}^{(i)} \) represents the deterministic gradient component, and \( \vn_{\vv}^{(i)} \) captures the propagated noise due to the channel. Here, \( \vn_{\vv}^{(i)} \) denotes the amplified channel noise at layer \( i \) for edge device \( k \).

\begin{figure}[t]
    \centering
    \includegraphics[width=0.49\textwidth]{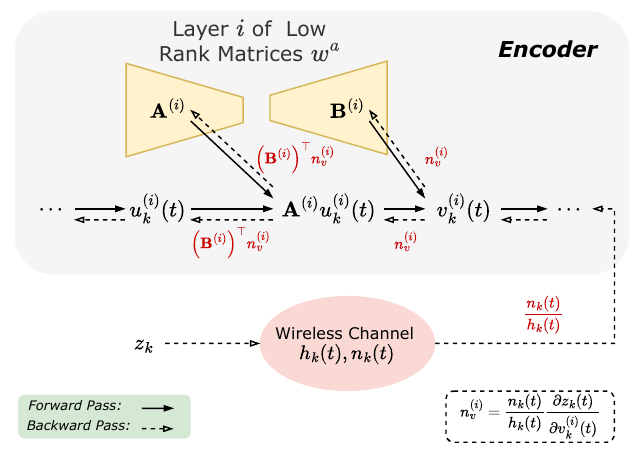}
    \caption{Computation graph of the \( i \)-th layer in the LoRA architecture during backpropagation, illustrating the propagation of gradient noise. The forward pass is represented by solid black arrows, while the backward pass is shown with dashed arrows. Noise introduced in the gradient during backpropagation is indicated by red text beneath the corresponding backward arrows. 
    }
    \label{fig:computation_graph}
\end{figure}

The gradients with respect to \( \mathbf{A}^{(i)} \) and \( \mathbf{B}^{(i)} \) are computed as
\begin{align}
    \nabla_{\mathbf{A}^{(i)}} \mathcal{L}_k(t) &= \left( \mathbf{B}^{(i)} \right)^\top \left( \vg_{\vv}^{(i)} + \vn_{\vv}^{(i)} \right) \left( \vu_k^{(i)}(t) \right)^\top, \\
    \nabla_{\mathbf{B}^{(i)}} \mathcal{L}_k(t) &= \left( \vg_{\vv}^{(i)} + \vn_{\vv}^{(i)} \right) \left( \vu_k^{(i)}(t) \right)^\top \left( \mathbf{A}^{(i)} \right)^\top.
\end{align}
The updates for \( \mathbf{A}^{(i)} \) and \( \mathbf{B}^{(i)} \) using a learning rate \( \eta \) are
\begin{align}
    \mathbf{A}^{(i)}_{\text{new}} &= \mathbf{A}^{(i)} - \eta \nabla_{\mathbf{A}^{(i)}} \mathcal{L}_k(t), \\
    \mathbf{B}^{(i)}_{\text{new}} &= \mathbf{B}^{(i)} - \eta \nabla_{\mathbf{B}^{(i)}} \mathcal{L}_k(t).
\end{align}
The change in the LoRA update is
\begin{equation}
    \Delta \vw^{(i)} = \mathbf{B}^{(i)}_{\text{new}} \mathbf{A}^{(i)}_{\text{new}}  - \mathbf{B}^{(i)}  \mathbf{A}^{(i)}.
\end{equation}

To streamline the notation in the following discussion,  we denote \( \mathbf{A}^{(i)} \) simply as \( \mathbf{A} \), and similarly drop the superscript \((i)\) for all vectors. Additionally, we will use \(\vu \) to denote \( \vu_k^{(i)}(t) \).
Expanding and simplifying the above equations, the change due to noise \( \vn_{\vv} \) can be expressed as
\begin{equation}\label{eq:change_due_to_noise}
    \begin{aligned}
        \vn_{\Delta \vw^{\rm c}} &= -\eta \left(\bB \bB^\top \vn_{\vv} \vu^\top + \vn_{\vv} \vu^\top \bA^\top \bA \right) \\
        &\quad + \eta^2 \left( T_1+T_2+T_3\right).
    \end{aligned}
\end{equation}
The higher-order term, scaled by $\eta^2$ involves the products of gradients and noise as follows,
\begin{equation}\label{eq:higher_order}
    \begin{aligned}
        T_1 &= (\vg_{\vv} \vu^\top \bA^\top)(\bB^\top \vn_{\vv} \vu^\top) \quad \text{(Mixed } \vg_{\vv} \text{ and } \vn_{\vv} \text{)}, \\
        T_2 &= (\vn_{\vv} \vu^\top \bA^\top)(\bB^\top \vg_{\vv} \vu^\top) \quad \text{(Mixed } \vn_{\vv} \text{ and } \vg_{\vv} \text{)},  \\
        T_3 &= (\vn_{\vv} \vu^\top \bA^\top)(\bB^\top \vn_{\vv} \vu^\top) \quad \text{(Quadratic in } \vn_{\vv} \text{)}.
    \end{aligned}    
\end{equation}
From \eqref{eq:higher_order}, we observe that the higher-order terms involve the quadratic and mixed products of $\vg_{\vv}$ and $\vn_{\vv}$. When the learning rate $\eta$ is large, these terms can significantly amplify the noise in the updates, which can adversely influence the convergence of the model and make training unstable.

\subsection{Redesigning the LoRA Architecture}
\label{subsec:redesign}

To solve the amplification of the noise
in \eqref{eq:change_due_to_noise}, we follow the strategy proposed in \cite{sun2024improving} and redesign the LoRA architecture. Specifically, we update only one of the low-rank matrices during training, i.e., matrix \( \mathbf{B} \), while keeping the other matrix \( \mathbf{A} \) fixed. 
With this modified setup, the noise in the gradient only propagates to update \( \mathbf{B} \), while \( \mathbf{A} \) remains constant. The weight update for layer \( i \) simplifies to
\begin{equation}
    \begin{aligned}
        \Delta \vw =  \mathbf{B}_{\text{new}} \mathbf{A}  - \mathbf{B}   \mathbf{A} = - \eta \nabla_{\mathbf{B} } \mathcal{L}_k(t) \mathbf{A} .
    \end{aligned}
\end{equation}
The change due to noise \( \vn_{\vv} \) can be simplified as
\begin{equation}
    \vn_{\Delta \vw^{\rm c}} = -\eta \vn_{\vv} \vu^\top \bA^\top \bA,
\end{equation}
while the energy of $\vn_{\Delta \vw^{\rm c}}$ can be write as
\begin{equation}\label{eq:covariance}
    E_{\vn_{\Delta \vw^{\rm c}}} = \tr(\eta^2 \bA^\top \bA \vu  \text{Cov}( \vn_{\vv})\vu^\top \bA^\top \bA),
\end{equation}
where $\tr$ denotes the trace of a matrix. To further minimize the impact of the amplified noise $\vn_{\vv}^{(i)}$, we initialize \(\mathbf{A} \) as an orthonormal matrix. In this case, \eqref{eq:covariance} can be futher simplified to
\begin{equation}\label{eq:simplecovariance}
    E_{\vn_{\Delta \vw^{\rm c}}} = \tr(\eta^2 \vu  \text{Cov}( \vn_{\vv})\vu^\top).
\end{equation}
Compared with \eqref{eq:covariance}, the simplified energy of $\vn_{\Delta \vw^{\rm c}}$ reduces the effect of the amplified channel noise $\vn_{\vv}^{(i)}$ during the training process, allowing the model to converge more effectively under a given privacy budget. The effectiveness of this approach will be demonstrated through experimental validation next.

\section{Simulation Results} \label{results}

{\color{black}In this section, we introduce the experiment setup in Section \ref{subsec::setup}, and present the experiment results and discussion in Section \ref{subsec::result}.}

\subsection{Experiment Setup}\label{subsec::setup}

Our primary objective is to evaluate the effectiveness of various methods under different privacy budgets. Specifically, we examine the following configurations,

\begin{enumerate}
    \item \textbf{Updating both low-rank matrices \( \mathbf{A} \) and \( \mathbf{B} \) (vanilla LoRA)}: Both matrices are updated during training, serving as the baseline.
    \item \textbf{Updating only \( \mathbf{B} \) with \( \mathbf{A} \) fixed as a Gaussian matrix}: \( \mathbf{A} \) is initialized with Gaussian values and kept fixed.
    \item \textbf{Updating only \( \mathbf{B} \) with \( \mathbf{A} \) initialized as a scaled orthonormal matrix (\( \mathbf{A}^\top \mathbf{A} = 0.01\mathbf{I} \))}: \( \mathbf{A} \) is initialized as a scaled orthonormal matrix to control spectral properties and remains fixed.
\end{enumerate}

We assign each edge device with 5\% data evenly sampled from the global dataset. The accuracy is  evaluated after each epoch to monitor convergence and performance. Other experimental settings are summarized in Table~\ref{tab:experiment-setup}. 
\begin{table}[htp]
    \centering
    \caption{Experimental Settings}
    \label{tab:experiment-setup}
    \begin{tabular}{l|l}
        \hline
        \textbf{Parameter} & \textbf{Value} \\
        \hline
        Dataset & CIFAR-10 \\
        Number of edge devices & 15 \\
        Privacy budget \( \varepsilon \) & \{3, 5, 10, 100\} \\
        DP Parameter \( \delta \) & \(1\times 10^{-5}\) \\
        Gradient clipping norm \( C \) & 0.01 \\
        Pretrained FM & Vision Transformer \cite{dosovitskiy2020vit} \\
        LoRA rank \( r \) & 4 \\
        Optimizer & Adam \\
        Learning rate & \(1 \times 10^{-3}\) \\
        Batch size & 32 \\
        Number of epochs & 30 \\
        \hline
    \end{tabular}
\end{table}

\subsection{\color{black}{Experiment Results and Discussion}}\label{subsec::result}

\begin{figure*}[ht]
    \centering
    \subfigure[$\varepsilon = 3$]{
        \includegraphics[width=0.36\textwidth]{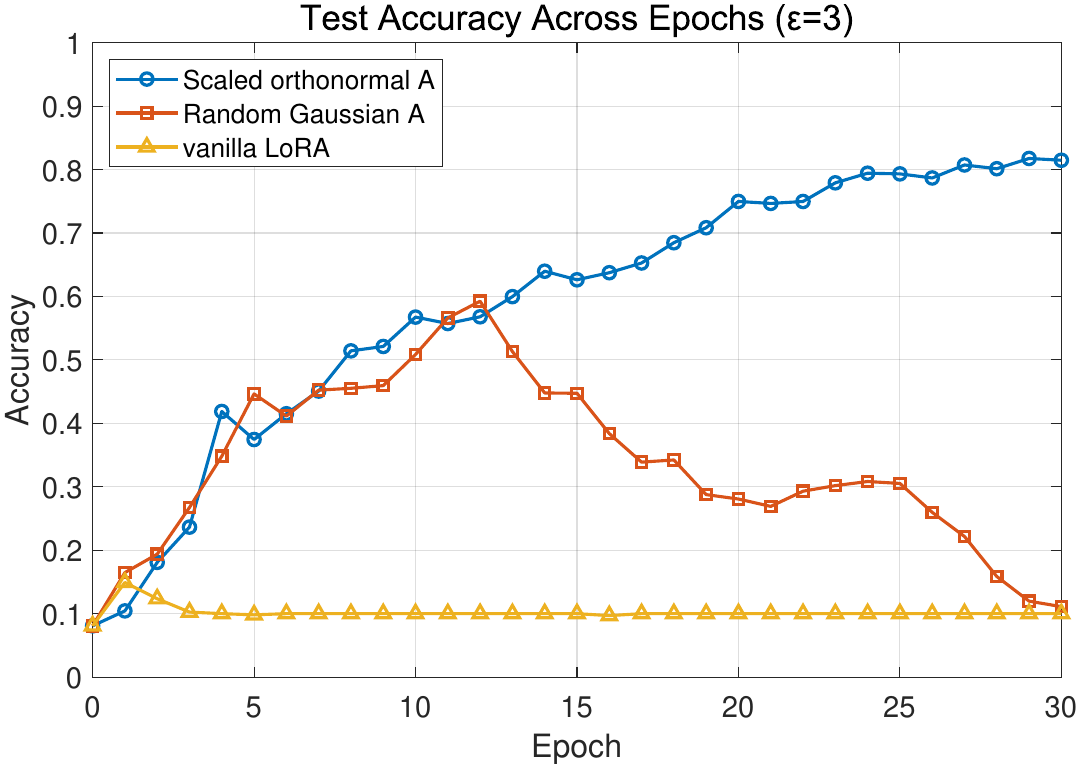}
        \label{fig:subfig1}
    }
    \subfigure[$\varepsilon = 5$]{
        \includegraphics[width=0.36\textwidth]{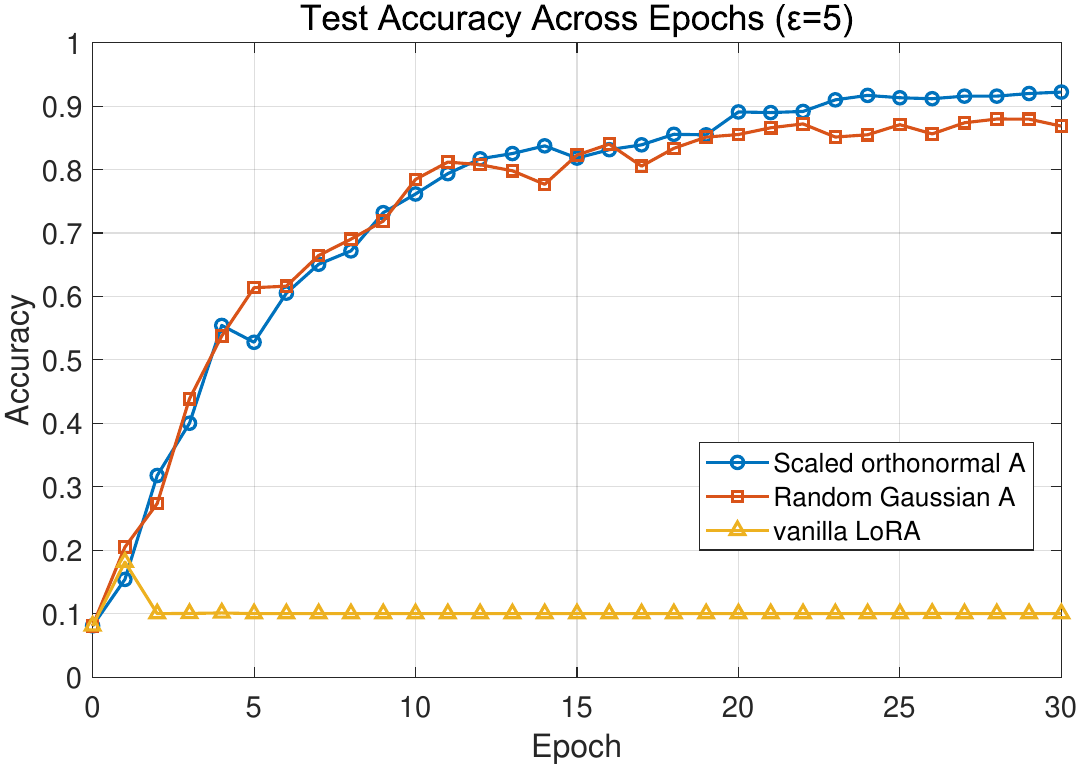}
        \label{fig:subfig2}
    }
    \subfigure[$\varepsilon = 10$]{
        \includegraphics[width=0.36\textwidth]{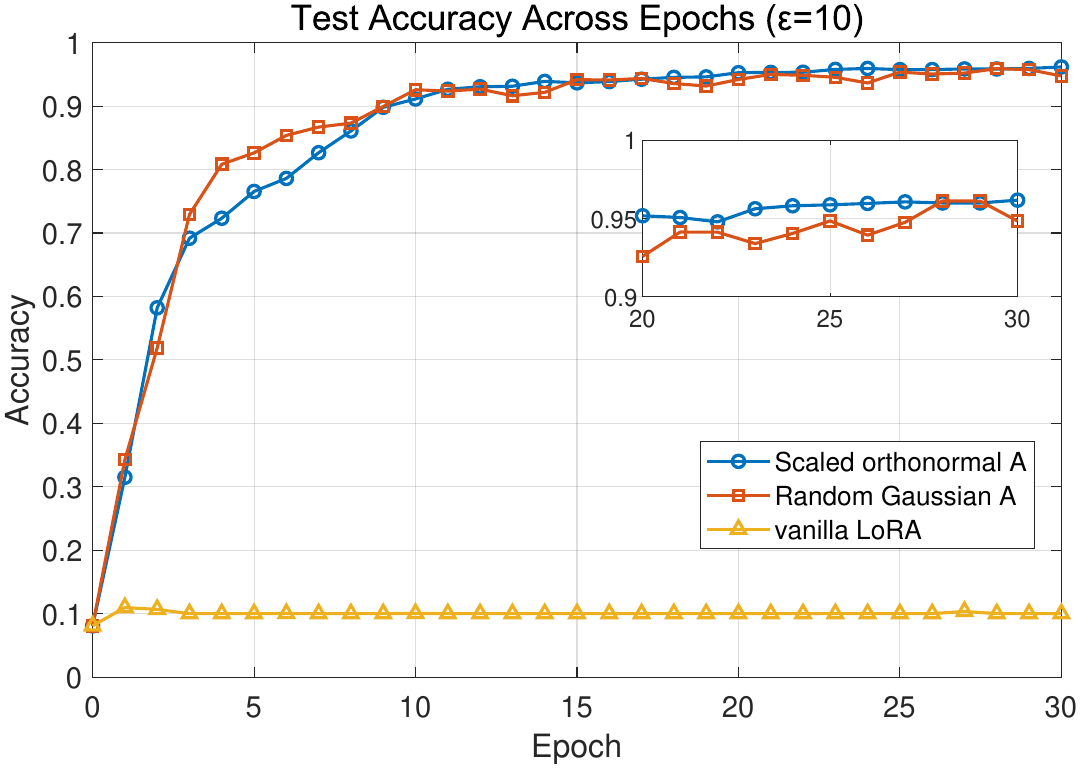}
        \label{fig:subfig3}
    }
    \subfigure[$\varepsilon = 100$]{
        \includegraphics[width=0.36\textwidth]{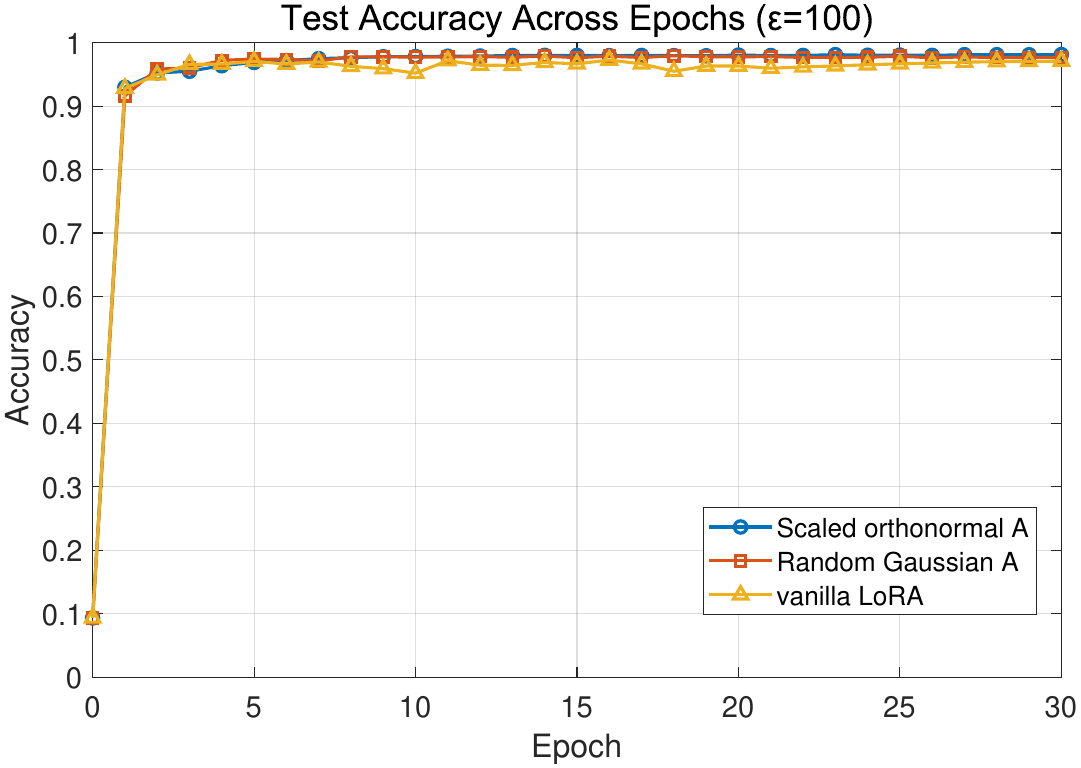}
        \label{fig:subfig4}
    }
    \caption{Test accuracy per epoch for different training configurations across varying privacy budgets (\( \varepsilon = 3 \), \( \varepsilon = 5 \), \( \varepsilon = 10 \), and \( \varepsilon = 100 \)). The configurations are: (1) updating both low-rank matrices \( \mathbf{A} \) and \( \mathbf{B} \) (vanilla LoRA), (2) updating only \( \mathbf{B} \) with \( \mathbf{A} \) fixed as a Gaussian matrix, and (3) updating only \( \mathbf{B} \) with \( \mathbf{A} \) initialized as a scaled orthonormal matrix.}
    \label{fig:privacy-utility}
\end{figure*}

Figure~\ref{fig:privacy-utility} presents the test accuracy over epochs for the three configurations under different privacy budgets \( \varepsilon \). As the privacy budget increases (i.e., privacy requirements are relaxed), the performance of all configurations improves due to the reduction of noise added for differential privacy.

At a strict privacy budget of \( \varepsilon = 3 \) (Fig.~\ref{fig:subfig1}), updating both \( \mathbf{A} \) and \( \mathbf{B} \) fails to learn meaningful features, with the test accuracy remaining close to random guessing. This outcome indicates that the noise introduced severely hampers the model's ability to converge when both matrices are updated, likely due to noise amplification in the cascaded architecture. Updating only \( \mathbf{B} \) with \( \mathbf{A} \) fixed as a Gaussian matrix shows initial improvement but diverges after a few epochs, suggesting that fixing \( \mathbf{A} \) reduces some noise amplification but does not sufficiently control the spectral properties to prevent divergence under strict privacy constraints. In contrast, our proposed method, updating only \( \mathbf{B} \) with \( \mathbf{A} \) initialized as a scaled orthonormal matrix, achieves stable convergence, reaching approximately 85\% test accuracy.

With a moderately relaxed privacy budget of \( \varepsilon = 5 \) (Fig.~\ref{fig:subfig2}), 
the Gaussian \( \mathbf{A} \) configuration improves to about 87\% but still lags behind our method. 

For \( \varepsilon = 10 \) (Fig.~\ref{fig:subfig3}), our method 
slightly surpassing the Gaussian \( \mathbf{A} \) configuration.
The configuration updating both \( \mathbf{A} \) and \( \mathbf{B} \) still fails to converge, emphasizing that noise amplification remains an issue even at this privacy level.
At the most relaxed privacy budget of \( \varepsilon = 100 \), 
updating both \( \mathbf{A} \) and \( \mathbf{B} \) now converges, indicating that at very high privacy budgets (i.e., minimal noise added), the impact of noise amplification is negligible.

The above results demonstrate that our proposed method effectively mitigates noise amplification by controlling the spectral properties of \( \mathbf{A} \), leading to superior performance, especially under strict privacy constraints. By updating only one low-rank matrix and initializing the other as a scaled orthonormal matrix, we enable stable and efficient learning in differentially private federated fine-tuning over wireless networks.

\section{Conclusions} \label{conclusion}
In this paper, we proposed an innovative federated fine-tuning framework that effectively integrates LoRA with DP in a split architecture suitable for resource-constrained edge devices. 
By leveraging the inherent wireless channel noise as a natural DP mechanism and updating only one low-rank matrix while fixing the other as an orthognal matrix, we mitigate noise amplification in the cascaded architecture and enhance training stability. 
Simulation results have demonstrated that the proposed method achieves superior performance over baselines under the same privacy constraints, enabling practical deployment of large-scale AI models on edge devices while ensuring robust privacy protection.

\bibliographystyle{IEEEtran}
\begingroup
\renewcommand{\baselinestretch}{0.97}
\bibliography{IEEEabrv,egbib_updated, ref}
\vspace{12pt}

\end{document}